\documentclass{llncs}
\usepackage{amssymb}
\usepackage{epsfig}
\usepackage{amsmath}

\newcommand{\nodes}{{\mathcal N}}
\newcommand{\edges}{{\mathcal E}}
\newcommand{\HG}{{\cal H}}
\newcommand{\HD}{H\!D}
\newcommand{\JT}{J\!T}

\newcommand{\NP}{\mbox{\rm NP}}

\newcommand{\LCFL}{\mbox{\rm LOGCFL}}

\newcommand{\DL}{{\rm L}}

\newcommand{\SL}{\mbox{\rm SL}}

\newcommand{\C}{\mathcal{C}}

\newcommand{\tuple}[1]{\langle#1\rangle}
\newcommand{\nop}[1]{}

\newcommand{\longv}[1]{}

\newcommand{\WCSP}{{\mbox{CSOP}}}
\newcommand{\WCSPP}{{\mbox{WCSP}}}
\newcommand{\VCSP}{{\mbox{Max-CSP}}}
\renewcommand{\P}{\mathcal{I}}
\newcommand{\I}{\mathcal{I}}
\newcommand{\primal}{{\it G}}
\newcommand{\incidence}{{\it inc}}
\newcommand{\w}{{w}}

\renewcommand{\th}{{t}}
\newcommand{\conforms}{\approx\ }

\newcommand{\rel}{r}

\begin{document}

\title{Tractable Optimization Problems through\\ Hypergraph-Based
  Structural Restrictions\thanks{G.Gottlob works at the Computing
    Laboratory and at the Oxford Man Institute of Quantitative
    Finance, Oxford University.  This work was done in the context of
    the EPSRC grant EP/G055114/1 ``Constraint Satisfaction for
    Configuration: Logical Fundamentals,Algorithms, and
    Complexity'' and of Gottlob's   Royal Society
    Wolfson Research Merit Award.}
%gg Se volete aggiungere qualcosa...
  }

\author{Georg Gottlob\inst{1}\and Gianluigi Greco\inst{2}\and Francesco Scarcello\inst{2}}

\institute{
  Oxford University\inst{1},
  University of Calabria\inst{2}\\
  {\tt georg.gottlob@comlab.ox.ac.uk}, {\tt ggreco@mat.unical.it}, {\tt scarcello@deis.unical.it}\\
}

\maketitle

\begin{abstract}
Several variants of the Constraint Satisfaction Problem have been proposed and investigated in the literature for modelling those scenarios
where solutions are associated with some given costs. Within these frameworks  computing an optimal solution is an $\NP$-hard problem in
general; yet, when restricted over classes of instances whose constraint interactions can be modelled via (nearly-)acyclic graphs, this problem
is known to be solvable in polynomial time.
In this paper, larger classes of tractable instances are singled out, by discussing solution approaches based on exploiting hypergraph
acyclicity and, more generally, structural decomposition methods, such as (hyper)tree decompositions.
\end{abstract}

%\noindent \textbf{Keywords:} CSPs, (hyper)graph acyclicity, computational complexity

\section{Introduction}\label{sec:introduction}

The Constraint Satisfaction Problem (CSP) is a well-known framework \cite{dech-03} for modelling and solving search problems, which received
considerably attention in the literature due to its applicability in various areas.
Informally, a CSP instance is defined by singling out the variables of interest, and by listing the allowed combinations of values for groups
of them, according to the constraints arising in the application at hand. The solutions for this instance are the assignments of domain values
to variables that satisfy all such constraints. Many apparently unrelated problems from disparate areas actually turn out to be equivalent to
the CSP and can be accommodated within the CSP framework. Examples are puzzles, conjunctive queries over relational databases, graph
colorability,
%$k$-SAT for fixed $k$,
and checking whether there is a homomorphism between two finite structures.

\begin{example}\label{ex:cross}
Figure~\ref{fig:crossword} shows a combinatorial crossword puzzle (taken from \cite{gott-etal-00}). A set of legal words is associated with each
horizontal or vertical array of white boxes delimited by black boxes. A solution to the puzzle is an assignment of a letter to each white box
such that to each white array is assigned a word from its set of legal words.
This problem can be recast in a CSP by associating a variable with each white box, and by defining a constraint for each array of white boxes
prescribing the legal words that are associated with it.\hfill $\lhd$
\end{example}

When assignments are associated with some given cost, however, computing an arbitrary solution might not be enough. For instance, the
crossword puzzle in Figure~\ref{fig:crossword} may admit more than one solution, and expert solvers may be asked to single out the most difficult ones, such as those solutions that
minimize the total number of vowels occurring in the used words.
In these cases, one is usually interested in the corresponding \emph{optimization problem} of computing the solution of minimum cost, whose
modeling is accounted for in several variants of the basic CSP framework, such as fuzzy, probabilistic, weighted, lexicographic, valued, and
semiring-based CSPs (see \cite{SoftConstraint,CSPext} and the references therein).

\begin{figure*}[t]
  \centering
  \includegraphics[width=0.99\textwidth]{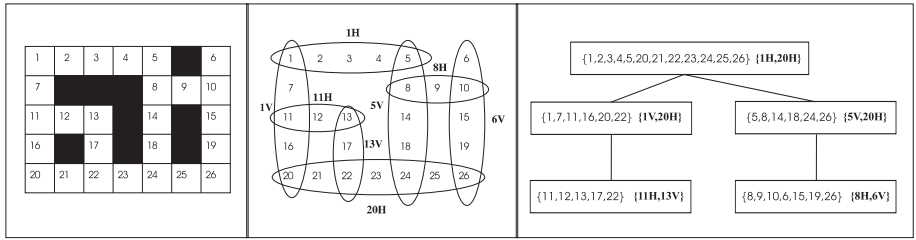}\vspace{-2mm}
  \caption{A crossword puzzle, its associated hypergraph $\HG_{cp}$, and a hypertree decomposition of width 2 for $\HG_{cp}$.}\label{fig:crossword}\vspace{-3mm}
\end{figure*}

Since solving CSPs---and the above extensions---is an $\NP$-hard problem, much research has been spent to identify restricted classes over
which solutions can efficiently be computed. In this paper, {\em structural decomposition methods} are considered \cite{gott-etal-00}, which
identify tractable classes by exploiting the structure of constraint scopes as it can be formalized either as a hypergraph (whose nodes
correspond to the variables and where each group of variables occurring in some constraint induce a hyperedge) or as a suitable binary encoding
of such hypergraph. In particular, we focus on the structural methods based on the notions of (generalized) hypertree
width~\cite{gott-etal-99,gott-etal-03} and treewidth~\cite{RS84}. In both cases, the underlying idea is that solutions to CSP instances that
are associated with acyclic (or nearly-acyclic) structures can efficiently be computed via dynamic programming, by incrementally processing the
structure according to some of its topological orderings.

As a matter of fact, however, while in the case of classical CSPs deep and useful results have been achieved for both graph and hypergraph
representations, in the case of CSP extensions tailored for optimization problems attention was mainly focused on binary encodings and, in
particular, on the \emph{primal graph} representation, where nodes correspond to variables and an edge between two variables indicates that
they are related by some constraint. Discussing whether (and how) hypergraph-based structural decomposition techniques in the literature can be
lifted to such optimization frameworks is the main goal of this paper.
In particular, we consider three CSP extensions:

\begin{description}
  \item[(1)] First, we consider optimization problems where every mapping variable-value is associated with a cost, so
      that the aim is to find an assignment satisfying all the constraints and having the minimum total cost.

  \item[(2)] Second, we consider the case where costs are associated with the allowed combinations of
      simultaneous values for the variables occurring in the constraint, rather than to individual values. Again, within this setting,
      we consider the problem of computing a solution having minimum total cost.

  \item[(3)] Finally, we consider a scenario where the CSP instance at hand might not admit a solution at all, and where the problem is
      hence to find the assignment minimizing the total number of violated constraints (and, more generally, whenever a cost is assigned to
      each constraint, the assignment minimizing the total cost of violated constraints).
\end{description}

\vspace{-2mm}For each of the above settings, the complexity of computing the optimal solution is analyzed in this paper, by overviewing some relevant recent
research and by providing novel results. In particular:

\vspace{-2mm}\begin{itemize}
  \item[$\blacktriangleright$] We show that optimization problems of kind \textbf{(1)} can be solved in polynomial time on instances
      having bounded (generalized) hypertree-width hypergraphs. This result is based on an algorithm recently designed and analyzed in the
      context of \emph{combinatorial auctions}~\cite{GG07}.

  \item[$\blacktriangleright$] We show that even optimization problems of kind \textbf{(2)} are tractable on instances having bounded (generalized) hypertree-width hypergraphs.
       Indeed, we describe how to transform this kind of instances into equivalent instances of kind \textbf{(1)}, by preserving their structural properties.

  \item[$\blacktriangleright$] We observe that optimization problems of kind \textbf{(3)} remain $\NP$-hard even over instances having an
      associated acyclic hypergraph. However, there is also good news:
      they are shown to be tractable on instances having bounded
      treewidth {\em incidence graph} encoding.
       The latter is a binary encoding of the constraint hypergraph with usually better structural features than the primal graph encoding (see, e.g.,
      \cite{gott-etal-00,grec-scar-03}). Again, proof is via a mapping to case \textbf{(1)}.
\end{itemize}

\vspace{-2mm}\noindent \textbf{Organization.} The rest of the paper is organized as follows. Section~\ref{sec:preliminaries} discusses preliminaries on CSPs
and structural restrictions, and Section~\ref{sec:methods} provides an overview of the structural decomposition methods based on treewidth and
(generalized) hypertree width. Results for optimization problems of kind \textbf{(1)} and \textbf{(2)} are discussed in Section~\ref{sec:uno},
whereas problems of kind \textbf{(3)} are discussed in Section~\ref{sec:tre}. Finally, Section~\ref{sec:extension} draws our conclusions.

\section{CSPs, Acyclic Instances, and their Desirable Properties}\label{sec:preliminaries}

An instance of a {constraint satisfaction problem} \cite{dech-03} is a triple $\P=\tuple{{\it Var},U,\C}$, where ${\it Var}$ is a finite set of
variables, $U$ is a finite domain of values, and $\C=\{C_1,C_2,\ldots,C_q\}$ is a finite set of constraints. Each constraint $C_v$, for $1\leq
v\leq q$, is a pair $(S_v,r_v)$, where $S_v\subseteq {\it Var}$ is a set of variables called the {\em constraint scope}, and $r_v$ is a set of
substitutions (also called \emph{tuples}) from variables in $S_v$ to values in $U$ indicating the allowed combinations of simultaneous values
for the variables in $S_v$. Any substitution from a set of variables $V\subseteq {\it Var}$ to $U$ is extensively denoted as the set of pairs
of the form $X/u$, where $u\in U$ is the value to which $X\in V$ is mapped. Then, a solution to $\P$ is a substitution $\theta:{\it Var}\mapsto
U$ for which $q$-tuples $t_1\in r_1,...,t_q\in r_q$ exist such that $\theta=t_1\cup...\cup t_q$.

\begin{example}
In the crossword puzzle of Figure~\ref{fig:crossword}, $\it Var$ coincides with the letters of the alphabet, and a variable $X_i$ (denoted by
its index $i$) is associated with each white box. An example of constraint is $C_{1H}=((1,2,3,4,5),r_{1H})$, and a possible instance for $r_{1H}$
is $\{ \tuple{h,o,u,s,e},\tuple{c,o,i,n,s},\tuple{b,l,o,c,k}\}$---in the various constraint names, subscripts $H$ and $V$ stand for
``Horizontal'' and ``Vertical,'' respectively, resembling the usual naming of definitions in crossword puzzles.~\hfill~$\lhd$
%
%
%%
%Moreover, the following constraints emerged from Figure~\ref{fig:crossword}: $C_{1H}=((1,2,3,4,5),r_{1H})$, $C_{8H}=((8,9,10),r_{8H})$,
%$C_{11H}=((11,12,13),r_{11H})$, $C_{20H}=((20,21,22,23,24,25,26),r_{20H})$, $C_{1V}=((1,7,11,16,20),$ $r_{1V})$,
%$C_{5V}=((5,8,14,18,24),r_{5V})$, $C_{6V}=((6,10,15,19,26),r_{6V})$, $C_{13V}=((13,17,22),r_{13V})$. Subscripts $H$ and $V$ stand for
%``Horizontal'' and ``Vertical,'' respectively, resembling the usual naming of definitions in crossword puzzles. A possible instance for the
%relation $r_{1H}$ is $\{ \tuple{h,o,u,s,e},\tuple{c,o,i,n,s},\tuple{b,l,o,c,k}\}$.
\end{example}

The structure of a CSP instance $\I$ is best represented by its associated hypergraph $\HG(\I)=(V,H)$,  where $V={\it Var}$ and $H=\{ S \mid
(S,r)\in {\mathcal C}\}$---in the following, $V$ and $H$ will be denoted by $\nodes(\HG)$ and $\edges(\HG)$, respectively.
As an example, the hypergraph associated with the crossword puzzle formalized above is illustrated in the central part of
Figure~\ref{fig:crossword}.

A hypergraph $\HG$ is {\em acyclic} iff it has a join tree~\cite{bern-good-81}. A {\em join tree} $\JT(\HG)$ for a hypergraph $\HG$ is a tree
whose vertices are the hyperedges of $\HG$ such that, whenever the same node $X\in V$ occurs in two hyperedges $h_1$ and $h_2$ of $\HG$, then
$X$ occurs in each vertex on the unique path linking $h_1$ and $h_2$ in $\JT(\HG)$. The notion of acyclicity we use here is the most general
one known in the literature, coinciding with $\alpha$-acyclicity according to Fagin~\cite{fagi-83}.
Note that the hypergraph $\HG_{cp}$ of Figure~\ref{fig:crossword} is not acyclic. An acyclic hypergraph is discussed below.

\begin{figure*}[t]
  \centering
  \includegraphics[width=0.99\textwidth]{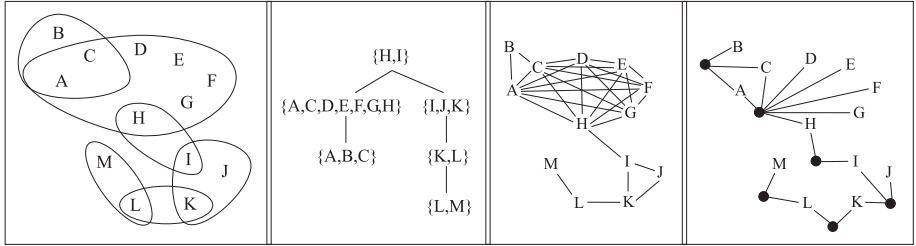}\vspace{-2mm}
  \caption{A hypergraph $\HG_1$, a join tree $\JT(\HG_1)$, the primal graph $\primal(\HG_1$), and the incidence
  graph $\incidence(\HG_1)$.}\label{fig:hypegraph}\vspace{-3mm}
\end{figure*}

\begin{example}
Consider the hypergraph $\HG_1$ shown on the left of Figure~\ref{fig:hypegraph}, which is associated with a CSP instance over the set of
variables $\{A,...,M\}$. In particular, six constraints are defined over the instance whose scopes precisely correspond to the hyperedges in
$\edges(\HG_1)$; for instance, $\{A,B,C\}$ is an example of constraint scope.
Note also that $\HG_1$ is acyclic. Indeed, a join tree $\JT(\HG_1)$ for it is reported in the same figure to the right of $\HG_1$.\hfill $\lhd$
\end{example}

An important property of acyclic instances is that they can efficiently be processed by dynamic programming. Indeed, according to Yannakakis'
algorithm~\cite{yann-81} (originally conceived in the equivalent context of evaluating acyclic Boolean conjunctive queries), they can be
evaluated by processing any of their join trees bottom-up, by performing upward semijoins between the constraint relations, thus keeping the
size of the intermediate result small. At the end, if the constraint relation associated with the root atom of the join tree is not empty, then
the CSP instance does admit a solution.
Therefore, the whole procedure is feasible in $O(n\times r_{max}\times \log r_{max})$,
where $n$ is
the number of constraints and $r_{max}$ denotes the size of the largest constraint relation.

In addition to the polynomial time algorithm for deciding whether a CSP admits a solution, acyclic instances enjoy further desirable
properties:

\vspace{-2mm}
\begin{description}
\item[Acyclicity is efficiently recognizable:] Deciding whether a hypergraph is acyclic is feasible in linear time \cite{tarj-yann-84} and
    belongs to the class $\DL$ (deterministic logspace). Indeed, this follows from the fact that hypergraph acyclicity belongs to $\SL$
    \cite{gott-etal-01}, and that $\SL$ is equal to $\DL$ \cite{rein-04}.

\item[Acyclic instances can be efficiently solved:] After the bottom-up step described above, one can perform the reverse top-down step
    by filtering each child vertex from those tuples that  do not match with its parent tuples.
    The relations obtained after the top-down step enjoy the {\em global consistency} property, i.e.,
    they contain only tuples whose values are part of some solution of the CSP.
    Then, all solutions can be computed with a backtrack-free
    procedure, and thus in {\em total polynomial time}, i.e., in time
    polynomial in the input plus the output~\cite{yann-81} (and
    actually also with polynomial delay). Alternatively, one may enforce pairwise consistency by taking the semijoins between all pairs of relations
    until a fixpoint is reached. Indeed, acyclic
    instances that fulfil this property also fulfil the global consistency property \cite{beer-etal-83}.

\item[Acyclic instances are parallelizable:] It has been shown that solving acyclic CSP instances is highly parallelizable, as this problem
    (actually, deciding the existence of a solution) is complete for the low complexity class $\LCFL$~\cite{gott-etal-01}. Efficient
    parallel algorithms are discussed in~\cite{gott-etal-01} and~\cite{gott-etal-98b}.
\end{description}

We conclude this section by recalling that the above desirable properties of acyclic CSP instances have profitably been exploited in various
application scenarios. Indeed, besides their application in the context of Database Theory, they found applications in Game
Theory~\cite{GGS05,DP06}, Knowledge Representation and
Reasoning~\cite{GPW06}, and Electronic Commerce \cite{GG07}, just to
name a few.

\section{Generalizing acyclicity} \label{sec:methods}

Many attempts have been made in the literature for extending the good results about acyclic instances to relevant classes of {\em nearly
acyclic} structures. We call these techniques {\em structural decomposition methods}, because they are based on the ``acyclicization'' of
cyclic (hyper)graphs.
We refer the interested reader to \cite{scar-etal-08} for a detailed description of how these techniques may be useful
for constraint satisfaction problems and to \cite{grec-scar-03} for further results about graph-based techniques, when relational structures
are represented according to various graph representations (primal graph, dual graph, incidence-graph encoding).
We also want to mention recent methods such as Spread-cuts~\cite{cohe-etal-06} and fractional hypertree decompositions~\cite{GroheMarx06}.

A survey of most of these techniques is currently available in Wikipedia (look for ``decomposition method'', at \url{http://www.wikipedia.org}).  In the sequel, we
shall briefly overview the tree and hypertree decomposition methods.

\subsection{Tree Decompositions}
\label{subsection:tw}
For classes of instances having only binary constraints or, more generally, constraints whose scopes have a fixed maximum arity, the most
powerful structural method is based on the notion of treewidth.

\begin{definition}[\cite{RS84}] \label{def:treewidth}\em
A {\em tree decomposition} of a graph $G=(V,E)$ is a pair $\tuple{T,\chi}$, where $T=(N,F)$ is a tree, and $\chi$ is a labelling function
assigning with each vertex $p\in N$ a set of vertices $\chi(p)\subseteq V$ such that the following conditions are satisfied: (1) for each node
$b$ of $G$, there exists $p\in N$ such that $b\in \chi(p)$; (2) for each edge  $(b,d)\in E$, there exists $p\in N$ such that $\{b,d\}\subseteq
\chi(p)$; and, (3) for each node $b$ of $G$, the set $\{p\in N \mid  b\in \chi(p)\}$ induces a connected subtree of $T$ (\emph{connectedness
condition}).
The {\em width} of $\tuple{T,\chi}$ is the number $\max_{p\in N}(|\chi(p)|-1)$. The {\em treewidth} of $G$, denoted by $tw(G)$, is the minimum
width over all its tree decompositions. \hfill $\Box$
\end{definition}

It is well-known that a graph $G$ is acyclic if and only if $tw(G)=1$. Moreover, for any fixed natural number $k>0$, deciding whether
$tw(G)\leq k$ is feasible in linear time \cite{bodl-96}.

Any CSP with primal graph $G$ such that $tw(G)\leq k$ can be (efficiently) turned into an equivalent CSP whose primal graph is acyclic. Let
$\P=\tuple{{\it Var},U,\C}$ be a CSP instance, let $G$ be the primal graph of $\HG(\P)$, and let $\tuple{T,\chi}$ be a tree decomposition of
$G$ having width $k$. We may build a new acyclic CSP instance $\P'=\tuple{{\it Var},U,\C'}$ over the same variables and universe as $\P$, but
with a different set of constraints $\C'$, as follows. Firstly, for each vertex $v$ of $T$, we create a constraint $(S_v,r_v)$, where
$S_v=\chi(v)$ and $r_v= U^{|\chi(v)|}$. Then, for every constraint $(S,r)\in \C$ of the original problem such that $S\subseteq \chi(v)$, we
eliminate from $r_v$ all those tuples that do not match with $r$. The resulting constraint is then added to $\C'$. It can be shown that $\P'$
has the same solutions as $\P$, and that it is acyclic. In fact, observe that, by construction, $\tuple{T,\chi}$ is a join tree of the
hypergraph $\HG(\P')$ associated with $\P'$, because of the connectedness condition of tree decompositions.
Furthermore, building $\P'$ from $\P$ is feasible in $O(n \times |U|^{k+1})$ where $n$ is the number of vertices in $T$, and where the size of
the largest constraint relation in the resulting instance is $|U|^{k+1}$. Since one can always consider only tree decompositions whose number
of vertices is bounded by the number of variables of the problem (i.e., the nodes of the graph), it follows that deciding whether $\P'$ (and
hence $\P$) is satisfiable is feasible in $O(|Var|\times|U|^{k+1} \times \log |U|^{k+1})$. In fact, as for acyclic instances, even in this case
we may compute also solutions for $\P$ with a backtrack-free search, after the preprocessing of the instance performed according to the given
tree decomposition (i.e., according to the join tree of the equivalent acyclic instance). As a consequence, all classes of CSP instances (with
primal graphs) having bounded treewidth may be solved in polynomial time, even if with an exponential dependency on the treewidth.

Clearly enough, this technique is not very useful for CSP instances with large constraint scopes. In particular, the class of CSP instances
whose associated constraint hypergraphs are acyclic are not tractable according to tree decompositions, because acyclic hypergraphs may have
unbounded treewidth. Intuitively, in  the primal graph all variables occurring in the same constraint scope are connected to
each other, and thus they lead to a clique in the graph. It follows that CSP instances having constraint scopes with large arities have large
treewidths, too, because the treewidth of a clique of $n$ nodes is $n-1$. As an example, Figure~\ref{fig:hypegraph} reports the graph
$\primal(\HG_1)$ associated with the acyclic hypergraph $\HG_1$, where one may notice how the hyperedge $\{A,C,D,E,F,G,H\}$ is flattened into a
clique over all its variables.

\subsection{Hypertree Decompositions}
\label{subsection:ht}

Let us now turn our attention to hypergraph based decompositions. Such decompositions are similar to tree decompositions, but they use an
additional covering of each set $\chi(p)$ with as few as possible hyperedges. The width is then no longer defined as the maximum cardinality of
$\chi(p)$ over all decomposition nodes $p$, but as the maximum number of hyperedges used to cover $\chi(p)$. Intuitively, this notion of width
is better, because it will allow us to expresses more accurately the computational effort needed to transform an instance into an acyclic one.

\begin{definition}[\cite{gott-etal-03}]\label{defn:genhyper-decomp}\em\
A {\em generalized hypertree decomposition} of a hypergraph $\HG$ is a triple $\HD=\tuple{T,\chi,\lambda}$, where $\tuple{T,\chi}$ is a tree
decomposition of the primal graph of $\HG$, and $\lambda$ is a labelling of the tree $T$ by sets of hyperedges of $\HG$ such that, for each
vertex $p\in vertices(T)$, $\chi(p)\subseteq \bigcup_{h\in \lambda(v)} h$. That is, all variables in the $\chi$ labeling are covered by
hyperedges (scopes) in the $\lambda$ labeling.
The {\em width} of $\HD$ is the number $\max_{p\in vertices(T)}(|\lambda(p)|)$. The {\em generalized hypertree width} of $\HG$, denoted by
$ghw(\HG)$, is the minimum width over all its generalized hypertree decompositions. If $I$ is a CSP instance then $ghw(I):=ghw({\HG}(I))$.\hfill
$\Box$
\end{definition}

Clearly, for each CSP instance $I$, $ghw(I)\leq tw(I)$. Moreover, there are classes of CSPs having unbounded treewidth whose generalized
hypertree width is bounded\cite{gott-etal-03}.

Finding a suitable tree decomposition whose sets $\chi(p)$ may each be covered with a few hyperedges seems to be quite a hard task even in case
we have some fixed upper bound $k$. Indeed, it has been shown that deciding whether $ghw(\HG)\leq k$ is $\NP$-complete (for any fixed $k\geq
3$)~\cite{GMS07}. Fortunately, since its first proposal in \cite{gott-etal-99}, this notion comes with a tractable variant, called hypertree
decomposition, whose associated width is at most 3 times (+1) larger than the generalized hypertree width~\cite{Galtro}. As a consequence, it
can be shown that every class of CSPs that is tractable according to generalized hypertree width is tractable according to hypertree width, as
well.

\begin{definition}[\cite{gott-etal-99}]\label{defn:hyper-decomp}\em\
A {\em hypertree decomposition} of a hypergraph $\HG$ is a generalized hypertree decomposition $\HD=\tuple{T,\chi,\lambda}$ that satisfies the
following additional condition, called {\it Descendant Condition} or also {\it special
condition}:
 $\forall p\in vertices(T)$, $\forall h\in \lambda(p)$, $h \cap \chi(T_p) \;\subseteq\; \chi(p)$,
 where $T_p$ denotes the subtree of $T$ rooted at $p$, and
  $\chi(T_p)$ the set of all variables occurring in the $\chi$ labeling of this subtree.

The {\em hypertree width} $hw(\HG)$ of $\HG$ is the minimum width over all its hypertree decompositions. \hfill $\Box$
\end{definition}

As an example, on the right part of Figure~\ref{fig:crossword} a hypertree decomposition of the hypergraph $\HG_{cp}$ in Example~\ref{ex:cross} is reported. Note that this decomposition has width 2.

We refer the interested reader to \cite{gott-etal-99,scar-etal-08} for more details about this notion and in particular about the descendant
condition. Here, we just observe that the notions of hypertree width and generalized hypertree width are true generalizations of acyclicity, as
the acyclic hypergraphs are precisely those hypergraphs having hypertree width and generalized hypertree width one. In particular, the classes
of CSP instances having bounded (generalized) hypertree width have the same desirable computational properties as
acyclic CSPs \cite{gott-etal-01}. Indeed, from a CSP instance $\P=\tuple{{\it Var},U,\C}$ and a (generalized) hypertree decomposition $\HD$ of
$\HG(\P)$ of width $k$, we may build an acyclic CSP instance $\P'=\tuple{{\it Var},U,\C'}$ with the same solutions as $\P$. The overall cost of
deciding whether $\P$ is satisfiable is in this case $O((m-1) \times r_{max}^k \times \log r_{max}^k)$, where $r_{max}$ denotes the size of the
largest constraint relation and $m$ is the number of vertices of the decomposition tree, with $m\leq |{\it Var}|$ (in that we may always find
decompositions in a suitable normal form without redundancies, so that the number of vertices in the tree cannot exceed the number of variables
of the given instance). To be complete, if the input consists of $\P$ only, we have to compute the decomposition, too. This can be done with a
guaranteed polynomial-time upper bound in the case of hypertree decompositions \cite{gott-etal-99}.

In the following two sections, we provide some tractability results for optimization problems. For the sake of presentation, we give algorithms
for the acyclic case, provided that these results may be clearly extended to any class of instances having bounded (generalized) hypertree
width, after the above mentioned polynomial-time transformation.

\section{Optimization Problems over CSP Solutions}\label{sec:uno}

In this section, we consider optimization problems where an assignment has to be singled out that satisfies all the constraints of the
underlying CSP instance and that has minimum total cost; in other words, we look for a ``best'' solution among all the possible solutions. In
particular, below, we shall firstly address the case where each possible variable-value mapping is associated with a cost (also called
constraint satisfaction optimization problem); then we shall consider the case where costs are defined over the constraints tuples (weighted
CSP).

\subsection{Constraint Satisfaction Optimization Problems}

An instance of a {\em constraint satisfaction optimization problem} ($\WCSP$) consists of a pair $\tuple{\P,\w}$, where $\P=\tuple{{\it
Var},U,\C}$ is a CSP instance and where $\w:{\it Var}\times U\mapsto \mathbb{Q}$ is a function mapping substitutions for individual variables
to rational numbers. For a substitution $\{X_1/u_1,...,X_n/u_n\}$, we denote by $\w(\{X_1/u_1,...,X_n/u_n\})$ the value $\sum_{i=1}^n
\w(X_i,u_i)$. Then, a solution to a $\WCSP$ instance $\tuple{\P,\w}$ is a solution $\theta$ to $\P$ such that $\w(\theta)\leq \w(\theta')$, for
each solution $\theta'$ to $\P$. Details on this framework can be found, e.g., in \cite{FCS}.

\begin{figure}[t]
\centering \fbox{
\parbox{0.91\textwidth}{\scriptsize
\begin{tabular}{l}
  \nop{\item}\textbf{Input}: An acyclic CSOP instance $\tuple{\P,\w}$ with $\P=\tuple{{\it Var},U,\C}$, $\C=\{(S_1,r_1),...,(S_q,r_q)\}$, \\
  \hspace{9mm}  and a join tree $T=(N,E)$ of the hypergraph $\HG(\P)$; \\
  \nop{\item}\textbf{Output}: A solution to $\tuple{\P,\w}$;\\
  \nop{\item}\textbf{var} \  $\th^*:{\it Var}\mapsto U$;\\
  \nop{\item} \hspace{4mm} $\ell_{{\th_v}}^v:\mbox{rational number, for each tuple } \th_v\in \rel_v$;\\
  \nop{\item} \hspace{4mm} $\th_{\th_v,c}:\mbox{tuple in }\rel_c$, for each tuple $\th_v\in \rel_v$, and for each $(v,c)\in E$;\\
  \nop{\item} \begin{tabular}{l}\hspace{-3mm}--------------------------------------------------------------------------------------------------------------------------\\
  \ \\
  \hspace{107mm}\ \end{tabular}\\
  \nop{\item}\ \vspace{-11mm}\\
  \nop{\item}\hspace{4mm}\\
  \begin{tabular}{l}
  \textbf{Procedure} $BottomUp$;\\
  \textbf{begin}\\
  \ \ \ $Done :=$ the set of all the leaves of $T$;\\
  \ \ \ \textbf{while} $\exists v\in T$ such that (i) $v\not\in Done$, and  (ii) $\{c \mid c \mbox{ is child of } v\}\subseteq Done$ \textbf{do}\\
%  \ \ \ \ \ \ \textbf{for each} $c$ such that $(v,c)\in E$ \textbf{do}\\
%  \ \ \ \ \ \ \ \ \ \ $H_v:=H_v -\{\h_v \mid \not\exists \h_c\in H_c \mbox{ s.t. } \h_v \conforms \h_c\}$;\\
  \ \ \ \ \ \ $\rel_v:=\rel_v-\{\th_v\mid \exists (v,c)\in E\mbox{ such that }\forall \th_c\in\theta_c,\ \th_v \not\conforms \th_c\}$;\\
  \ \ \ \ \ \ \textbf{if} $\rel_v=\emptyset$ \textbf{then} EXIT;\ \ \  (* \texttt{$\P$ is not satisfiable} *)\\
  \ \ \ \ \ \ \textbf{for each} $\th_v\in \rel_v$ \textbf{do}\\
  \ \ \ \ \ \ \ \ \ \ $\ell_{\th_v}^v:=w(\th_v)$;\\ %\ $\h'_v:=\h_v$;\\
  \ \ \ \ \ \ \ \ \ \ \textbf{for each} $c$ such that $(v,c)\in E$ \textbf{do}\\
  \ \ \ \ \ \ \ \ \ \ \ \ \ \ $\bar \th_c:=
  \arg \min_{\th_c\in \rel_c \mid \th_v \conforms \th_c} \left(\ell_{\th_c}^c- w(\th_c\cap\th_v)\right);$\\
  \ \ \ \ \ \ \ \ \ \ \ \ \ \ $\th_{\th_v,c}:=\bar \th_c$;\ \ \  (* \texttt{set best solution} *)\\
  \ \ \ \ \ \ \ \ \ \ \ \ \ \ $\ell_{\th_v}^v:=\ell_{\th_v}^v+\ell_{\bar \th_c}^c- w(\bar \th_c\cap \th_v)$;\\ %\ $\h'_v:=\h'_v\mbox{- }\bar \h_c$;\\
  \ \ \ \ \ \ \ \ \ \ \textbf{end for}\\
 % \ \ \ \ \ \ \ \ \ \ $\ell_{\h_v}:=\ell_{\h_v}+w(\h_v)$;\\
  \ \ \ \ \ \ \textbf{end for}\\
  \ \ \ \ \ \ $Done:=Done \cup \{v\}$;\\
  \ \ \ \textbf{end while}\\
  \textbf{end;}
  \end{tabular}\\
  \nop{\item}\begin{tabular}{l}\hspace{-3mm}--------------------------------------------------------------------------------------------------------------------------\\
  \ \\
  \hspace{112mm}\ \end{tabular}\\
  \nop{\item}\ \vspace{-11mm}\\
  \nop{\item}\\
  \hspace{-1mm}\begin{tabular}{ll|l}
  \begin{tabular}{l}
  \nop{\item}\textbf{begin} (* MAIN *)\\
  \nop{\item} \ \ \ \ $BottomUp$;\\
  \nop{\item} \ \ \ \ let $r$ be the root of $T$;\\
  \nop{\item} \ \ \ \ $\bar \th_r:=\arg \min_{\th_r \in \rel_r} \ell_{\th_r}^r$;\\
  \nop{\item} \ \ \ \ $\th^*:=\bar \th_r$;\ \ \  (* \texttt{include solution} *)\\
  \nop{\item} \ \ \ \ $TopDown(r,\bar \th_r)$;\\
  \nop{\item} \ \ \ \ \textbf{return} $\th^*$;\\
  \nop{\item}\textbf{end}.\\
  \end{tabular}&\hspace{-1mm} & \hspace{0mm}
  \begin{tabular}{l}
  \vspace{-6mm}\nop{\item}\ \\ \ \\
  \textbf{Procedure} $TopDown(v:\mbox{vertex of } N$, $\th_v \in r_v$); \\
%  \hspace{3mm} $h_v:\mbox{ partial packing for }\lambda(v))$;\\
  \textbf{begin}\\
  \ \ \ \textbf{for each} $c\in N$ s.t. $(v,c)\in E$ \textbf{do}\\
  \ \ \ \ \ \ \ $\bar \th_c:=\th_{\th_v,c}$;\\
  \ \ \ \ \ \ \ $\th^*:=\th^*\cup \bar \th_c$;\ \ \  (* \texttt{include solution} *)\\
  \ \ \ \ \ \ \ $TopDown(c,\bar \th_c$);\\
  \ \ \ \textbf{end for}\\
  \textbf{end;}
  \end{tabular}
  \end{tabular}
\end{tabular}
}}
 \vspace{-1mm}\caption{\textbf{Algorithm} {\sc ComputeOptimalSolution}.} \label{fig:algoritmo2}\vspace{-2mm}
\end{figure}

Constraint satisfaction optimization problems naturally arise in various application contexts. As an example they have recently been used in
the context of combinatorial auctions~\cite{GG07}, in order to model and solve the \emph{winner determination problem} of determining the
allocation of the items among the bidders that maximizes the sum of the accepted bid prices. In particular, in \cite{GG07}, it has been
observed that CSOPs and, in particular, the winner determination problem, can be solved in polynomial time on
%gg certain classes? which precisely?
%ggreco
%certain: e' troppo complicato spiegare la questione del dual hypergraph forse. Lascerei informale, con certain commentato
%fs Metterei comunque qualcosa, anche "some" se non "certain"
%ggreco: messo "some"
some classes of acyclic instances via a dynamic programming algorithm founded on the ideas of \cite{yann-81}. This algorithm, named {\sc
ComputeOptimalSolution}, is reported in Figure~\ref{fig:algoritmo2} and will be briefly illustrated in the following.

The algorithm receives in input the instance $\tuple{\P,\w}$ and a join tree $T=(N,E)$ for $\HG(\P)$. Recall that each vertex $v\in N$
corresponds to a hyperedge of $\HG(\P)$ and, in its turn, to a constraint in $\C$; hence, we shall simply denote by $(S_v,r_v)$ the constraint
in $\C$ univocally associated with vertex~$v$.

Based on $\tuple{\P,\w}$ and $T$, {\sc ComputeOptimalSolution} computes an optimal solution (or checks that there is no solution) by looking
for the ``conformance'' of the tuples in each relation $\rel_v$ with the tuples in $\rel_c$, for each child $c$ of $v$ in $T$, where $\th_v\in
r_v$ is said to \emph{conform} with $\th_c\in r_c$, denoted by $\th_v \conforms \th_c$, if for each $X\in S_v\cap S_c$, $X/u \in \th_v
\Leftrightarrow X/u \in \th_c$.
In more detail, {\sc ComputeOptimalSolution} solves $\tuple{\P,\w}$ by traversing $T$ in two phases. First, vertices of $T$ are processed from
the leaves to the root $r$, by means of the procedure $BottomUp$ that updates the weight $\ell_{{\th_v}}^v$ of the current vertex $v$.
Intuitively, $\ell_{{\th_v}}^v$ stores the cost of the best partial solution for $\P$ computed by using only the variables occurring in the
subtree rooted at $v$. Indeed, if $v$ is a leaf, then $\ell_{{\th_v}}^v=w(\th_v)$. Otherwise, for each child $c$ of $v$ in $T$,
$\ell_{{\th_v}}^v$ is updated by adding the minimum value $\ell_{\th_c}^c- w(\th_c\cap \th_v)$ over all tuples $\th_c$ conforming with $\th_v$.
The tuple $\bar \th_c$ for which this minimum is achieved is stored in the variable $\th_{\th_v,c}$ (resolving ties arbitrarily). Note that if
this process cannot be completed, because there is no tuple in $r_v$ conforming with some tuple in each relation associated with the children
of $v$, then we may conclude that $\P$ is not does not admit any solution. Otherwise, after the root $r\in N$ is reached, this part ends, and
the top-down phase may start.

In this second phase, the tree $T$ is processed starting from the root. Firstly, the assignment $\th^*$ is defined as the tuple in $\rel_r$
with the minimum cost over all the tuples in $\rel_r$ (again, resolving ties arbitrarily). Then, procedure $TopDown$ extends $\th^*$ with a
tuple for each vertex of $T$: at each vertex $v$ and for each child $c$ of $v$, $\th^*$ is extended with the tuple $\th_{\th_v,c}$ resulting
from the bottom-up phase.

Being based on a standard dynamic programming scheme, correctness of {\sc ComputeOptimalSolution} can be shown by structural induction on the
subtrees of $T$ \cite{GG07}. Moreover, by analyzing its running time, one may note that dealing with cost functions does not (asymptotically)
provide any overhead w.r.t. Yannakakis's algorithm \cite{yann-81} for plain CSPs. Following~\cite{GG07}, the following can be shown for the
more general case of $\WCSP$ instances having bounded generalized hypertree-width hypergraphs.\footnote{In all complexity results, we assume
the weighting function $\w$ be explicitly listed in the input (otherwise, just add the cost of computing through $\w$ all cost values for the
variable assignments of the given input instance).}

\begin{theorem}\label{thm:teo1}
Let $\tuple{\P,\w}$ be a {\em $\WCSP$} instance and $\HD$ a (generalized) hypertree decomposition of $\HG(\P)$. Moreover, let $k$ be the width
of $\HD$ and $m$ be the number of vertices in its decomposition tree. Then, a solution to $\tuple{\P,\w}$ can be computed (or it is discovered
that no solution exists) in time $O((m-1)\times r_{max}^k\times \log r_{max}^k)$, where $r_{max}$ is the size of the largest constraint
relation in $\P$.
\end{theorem}

\subsection{Weighted CSPs: Costs over Tuples}\label{sec:WCSP}

Let us now turn to study a slight variation of the above scenario, where costs are associated with each {\em tuple} of the constraint
relations, rather than with substitutions for individual variables. In fact, this is the setting of weighted CSPs, a well-known specialization
of the more general \emph{valued} CSP framework \cite{SFV95}.

Formally, a {\em weighted} CSP ($\WCSPP$) instance consists of a tuple $\tuple{\P,\w_1,...,\w_q}$, where $\P=\tuple{{\it Var},U,\C}$ with
$\C=\{C_1,C_2,\ldots,C_q\}$ is a CSP instance, and where, for each tuple $t_v\in r_v$, $\w_v(t_v)\in \mathbb{Q}$ denotes the cost associated
with $t_v$. For a solution $\theta=t_1\cup...\cup t_q$ to $\P$, we define $\w(\theta)=\sum_{v=1}^q \w_v(t_v)$ as its associated cost. Then, a
solution to $\tuple{\P,\w_1,...,\w_q}$ is a solution $\theta$ to $\P$ such that $\w(\theta)\leq \w(\theta')$, for each solution $\theta'$ to
$\P$.

A few tractability results for $\WCSPP$s (actually, for valued CSPs) are known in the literature when structural restrictions are considered
over binary encodings of the constraint hypergraphs. Indeed, it has been observed that $\WCSPP$s are tractable when restricted on classes of
instances whose associated primal graphs are acyclic or nearly-acyclic (see, e.g., \cite{TJ03,GSV06,NJT08}). However, the primal graph obscures
much of the structure of the underlying hypergraph since, for instance, each hyperedge is turned into a clique there---see the discussion in
Section~\ref{sec:methods}.

Therefore, whenever constraints have large arities, tractability results for primal graphs are useless, and it becomes then natural to ask
whether polynomial-time solvability still holds when moving from (nearly-)acyclic primal graphs to acyclic hypergraphs, possibly associated
with very intricate primal graphs. Next, we shall positively answer this question, by simply recasting weighted CSPs as constraint optimization
problems, and by subsequently solving them via the algorithm {\sc ComputeOptimalSolution}.
To this end, given a $\WCSPP$ instance $\tuple{\P,\w_1,...,\w_q}$, we define its associated $\WCSP$ instance, denoted by
$\WCSP(\tuple{\P,\w_1,...,\w_q})$, as the pair $\tuple{\P',w'}$ with $\P'=\tuple{{\it Var}',U',\C'}$ such that:

\vspace{-1mm}\begin{itemize}
  \item[$\bullet$] ${\it Var}'={{\it Var}\cup \{D_1,...,D_q\}}$, where each $D_v$ is a fresh auxiliary variable in $\P'$;

  \item[$\bullet$] $U'=U\cup\bigcup_{v=1}^q\bigcup_{t_v\in r_v} \{u_{t_v}\}$, i.e., for each constraint $(S_v,r_v)\in \C$, $U'$ contains a
  fresh value for each tuple in $r_v$---intuitively, mapping the variable $D_v$ to $u_{t_v}$ encodes that the tuple $t_v$ is going to contribute to a solution for $\P$;

  \item[$\bullet$] $\C'=\{ (S_v\cup\{D_v\},r_v') \mid  (S_v,r_v)\in \C\}$, where $r_v' = \{ t_v\cup \{D_v/u_{t_v}\} \mid t_v\in r_v \}$;

  \item[$\bullet$] $w'(X/u)=w_v(t_v)$ if $X=D_v$ and $u=u_{t_v}$, for some tuple $t_v\in r_v$; otherwise, $w'(X/u)=0$. That is,
  the whole cost of each tuple is determined by the mapping of its associated fresh variable $D_v$.
\end{itemize}

It is immediate to check that the above transformation is feasible in linear time. In addition, the transformation enjoys two relevant
preservation properties: Firstly, it preserves the structural properties of the $\WCSPP$ instance in that $\HG(\P')$ is acyclic if and only if
$\HG(P)$ is acyclic; and secondly, it preserves its solutions, in that $\theta'=t_1'\cup...\cup t_q'$ is a solution to
$\tuple{\P,\w_1,...,\w_q}$ if and only if $\theta=t_1\cup...\cup t_q$ is a solution to $\tuple{\P',w'}$, where $t_v'=t_v\cup \{D_v/u_{t_v}\}$
for each $1\leq v\leq q$. By exploiting these observations and Theorem~\ref{thm:teo1}, the following can be established.

\begin{theorem}
Let $\tuple{\P,\w_1,...,\w_q}$ be a {\em $\WCSPP$} instance and $\HD$ a (generalized) hypertree decomposition of $\HG(\P)$. Moreover, let $k$
be the width of $\HD$ and $m$ be the number of vertices in its decomposition tree. Then, a solution to $\tuple{\P,\w_1,...,\w_q}$ can be
computed (or we may state that there is no solution) in time $O((m-1)\times r_{max}^k\times \log r_{max}^k)$, where $r_{max}$ is the size of
the largest constraint relation in $\P$.
\end{theorem}

\section{Minimizing the Number of Violated Constraints}\label{sec:tre}

In this section, we shall complete our picture  by considering those scenarios where problems might possibly be overconstrained and where,
hence, the focus is on finding assignment minimizing the total number of violated constraints. These kinds of problems are usually referred to
in the literature as $\VCSP$s~\cite{FW92}, which similarly as $\WCSPP$s are specializations of valued CSPs.

Formally, let $\theta:{\it Var}\mapsto U$ be an assignment for a CSP instance $\P=\tuple{{\it Var},U,\C}$. We say that the violation degree of
$\theta$, denoted by $\delta(\theta)$, is the number of relations $r_v$ such that there is no tuple $t_v\in r_v$ with $t_v\subseteq \theta$. An
assignment $\theta:{\it Var}\mapsto U$ is a solution to the $\VCSP$ instance (associated with $\P$) if $\delta(\theta)\leq \delta(\theta')$,
for each assignment $\theta':{\it Var}\mapsto U$. Note that $\VCSP$s instances, by definition, do always have a solution.

\subsection{Acyclic Instances Remain Intractable}

After the tractability results established in Section~\ref{sec:WCSP} for $\WCSPP$s, one may expect good news for  $\VCSP$s, too. Surprisingly,
this is not the case.

\begin{theorem}
Solving {\em $\VCSP$s} is $\rm NP$-hard, even when restricted over classes of instances with acyclic constraint hypergraphs.
\end{theorem}
\begin{proof}
Consider any class $\mathcal{T}$ of CSPs instances having an $\NP$-hard satisfiability problem. Then, let $\mathcal{T}'$ be a new class of
$\VCSP$ instances such that, for each $\I=\tuple{{\it Var},U,\C}\in\mathcal{T}$, $\mathcal{T}'$ contains an instance $\I'=\tuple{{\it
Var},U,\C'}$ with $\C'=\C\cup \{({\it Var},\emptyset)\}$. That is, any instance $\I'\in \mathcal{T}'$ has a constraint over all variables with
an empty constraint relation, and thus it is not satisfiable. Moreover, because of the big hyperedge associated with such a constraint, its
hypergraph $\HG(\I')$ is trivially acyclic. Also, by construction, there is an assignment for $\I'$ violating only one constraint if and only
if $\I$ is satisfiable. It follows that finding an assignment minimizing the total number of violated constraints is $\NP$-hard on the class of
acyclic instances $\mathcal{T}'$. \hfill $\Box$
\end{proof}

%gg Il seguente confonde troppo le idee del lettore.
\nop{
Note that the above proof of intractability is based on a big hyperedge covering all variables that is associated with an unsatisfiable
constraint. Of course, this construction works because, in the considered framework, constraint relations are encoded by listing just the
allowed tuples of values. In fact, in other settings where all tuples have to be listed (both good and bad combinations of values)---see,
e.g.,~\cite{KDLD05}, the above construction does not work because building the relation for the big-hyperedge constraint would take exponential
time (and space). On the other hand, in such frameworks it is hard to get benefits by exploiting hypergraph-based methods, because the sizes of
the input relations are always exponential in their arities, and thus it may be convenient to use simpler graph-based techniques.
}

\subsection{Incidence graphs and Tractable Cases}

Given that hypergraph acyclicity and hence its generalizations are not sufficient for guaranteeing the tractability of Max-CSPs, it makes sense
to explore acyclicity properties related to suitable graph representations. In fact, as observed in Section~\ref{sec:WCSP}, it is well-known
that valued CSPs (and, hence, $\VCSP$s) are tractable over acyclic primal graphs (e.g., \cite{TJ03,GSV06,NJT08}). More precisely, tractability
has been observed in the literature to hold over primal graphs having bounded \emph{treewidth} (see Section~\ref{sec:methods}). Our main result
in this section is precisely to show that tractability still holds in case the  \emph{incidence graph} of $\HG(\P)$ has bounded treewidth,
which is a more general condition than the bounded treewidth of  primal graphs and which can be used to establish better complexity bounds and
to enlarge the class of tractable instances~\cite{grec-scar-03}. The fact that the standard CSP is tractable for instances whose incidence
graphs have bounded treewidth was already shown in~\cite{chek-raja-00}. We here extend this tractability result to Max-CSPs.

Recall that the incidence encoding of a hypergraph $\HG$, denoted by $\incidence(\HG)=(N,E)$, is the bipartite graph where $N=\edges(\HG)\cup
\nodes(\HG)$ and $E=\{\ \{h,a\}\mid h \in \edges(\HG) \mbox{ and } a \in h)\}$, i.e. it contains an edge between $h$ and $a$ if and only if the
variable $a$ occurs in the hyperedge $h$. As an example, Figure~\ref{fig:hypegraph} reports on the rightmost part the incidence graph
$\incidence(\HG_1)$, where nodes associated with hyperedges in $\edges(\HG_1)$ are depicted as black circles. Note that the treewidth of
$\incidence(\HG_1)$ is 2, which is much smaller than the treewidth of $\primal(\HG_1)$. This does not happen by chance since, for each
hypergraph $\HG$, it holds that $tw(\incidence(\HG))\leq tw(\primal(\HG))$; in addition, there are also classes of hypergraphs with incidence
encodings of bounded treewidth and primal encodings of unbounded treewdith (see, e.g., \cite{grec-scar-03}).

While enlarging the class of instances having bounded treewidth, the incidence encoding still conveys all the information needed to solve
$\VCSP$ instances. Again, the solution algorithm consists of a transformation into a suitable $\WCSP$ instance. Formally, let $\I=\tuple{{\it
Var},U,\C}$ be a $\VCSP$ instance with $\C=\{(S_1,r_1),...,(S_q,r_q)\}$, and let $\tuple{T,\chi}$ be a $k$-width tree decomposition of
$\incidence(\HG(\I))$---recall that for each vertex $v\in T$, $\chi(v)$ is a set of variables (i.e., nodes of $\nodes(\HG(\I))$) and constraint
scopes (i..e, edges in $\edges(\HG(\I))$). Then, the constraint satisfaction optimization problem instance $\WCSP(\P,\tuple{T,\chi})$ is the
pair $\tuple{\P',w'}$, where $\P'=\tuple{{\it Var}',U',\C'}$ and such that:

\vspace{-1mm}\begin{itemize}
  \item[$\bullet$]  ${\it Var}'={\it Var}\cup \{S_1,...,S_q\}$, that is, also the constraint scopes of $\C$ belong to the variables of the new problem;

  \item[$\bullet$] $U'=U\cup\{\mathit{unsat}\}\cup \{ u_t \mid t\in r_i, \mbox{ for }
  1\leq i\leq q\}$;

  \item[$\bullet$] $\C'=\{ (\chi(v),r'_v) \mid  v\in T\}$ where the constraint relation $r'_v$ is defined as follows.
  Let $\mu = |\chi(v)\cap{\it Var}|$, and let $U^\mu$ denote the set of all possible tuples over the $\mu$ variables in $\chi(v)\cap{\it Var}$.
  Let also $S_{i_1},...S_{i_h}$ be the scope-variables in $\chi(v)$. Then, for each tuple $\theta\in U^\mu$, the relation $r'_v$
  contains all tuples $\theta\cup \{S_{i_1}/v_{i_1}\}\cup \cdots\cup \{ S_{i_h}/v_{i_h} \}$, where $v_{i_j}\in U$ ($1\leq j\leq h$) is a value
  for the scope-variable $S_{i_j}$ such that: $v_{i_j}=u_t$ if there is a tuple $t\in r_{i_j}$ conforming with $\theta$; and $v_{i_j}=\mathit{unsat}$,
  if no such a tuple exists in $r_{i_j}$.

  \item[$\bullet$] $w'(X/u)=0$ if $u\neq \mathit{unsat}$; otherwise $w'(X/u)=1$,
   that is, each constraint of $\C$ that is not satisfied increases the cost of a
  solution by a unitary factor.
\end{itemize}

Note that this transformation is feasible in time exponential in the width of $\tuple{T,\chi}$ only. Moreover, solutions of $\P'$ with minimum
total cost precisely correspond to assignments over $\P$ minimizing the total number of violated constraints. In fact, the following can be
established.

\begin{theorem}
Let $\P=\tuple{{\it Var},U,\C}$ be a {\em $\VCSP$} instance with $tw(\incidence(\HG(\P)))= k$. Then, a solution to $\P$ can be computed in time
$O(|{\it Var}|\times |U|^{k+1}\times \log |U|^{k+1})$.
\end{theorem}

\section{Conclusion and Discussion}\label{sec:extension}

In this paper, classes of tractable $\WCSP$, $\WCSPP$, and $\VCSP$ instances are singled out by overviewing and proposing solution approaches
applicable to instances whose hypergraphs have bounded (generalized) hypertree width, or whose incidence graphs have bounded treewidth. The
techniques described in this paper are mainly based on Algorithm~{\sc ComputeOptimalSolution}, which has been designed to optimize costs
expressed as rational numbers and combined via the summation operation. However, it is easily seen that it remains correct if costs are
specified over an arbitrary totally ordered monoid structure, where some binary operation $\oplus$ (in place of standard summation) is used in
order to combine costs, provided it is commutative, associative, closed, and that it verifies identity and monotonicity. It follows that all
tractable classes of $\WCSP$, $\WCSPP$, and $\VCSP$ instances identified in this paper remain tractable in such extended scenarios, which
indeed emerge with valued CSPs (see, e.g., \cite{CSPext}).

\end{document}